\def\year{2016}\relax
\documentclass[letterpaper]{article}
\usepackage{aaai16}
\usepackage{times}
\usepackage{helvet}
\usepackage{courier}
\DeclareSymbolFont{bbold}{U}{bbold}{m}{n}
\DeclareSymbolFontAlphabet{\mathbbold}{bbold}
\usepackage{enumerate}
\usepackage{amsmath}
\usepackage{amsfonts}
\usepackage{amssymb}
\usepackage{amsbsy}
\usepackage{bbm}
\usepackage{isomath}
\usepackage{amsthm}
\usepackage{dsfont}
\usepackage{algorithm}
\usepackage{algpseudocode}
\usepackage{mathrsfs}
\usepackage{paralist}
\usepackage{epsfig}
\usepackage{subfig}
\usepackage{makeidx} 
\usepackage{mathtools}
\usepackage{comment}
\usepackage{breqn}
\usepackage{pgf}
\usepackage{tikz}

\numberwithin{equation}{section}

\usetikzlibrary{external}
\theoremstyle{plain}
\newtheorem{corollary}{Corollary}
\newtheorem{lemma}{Lemma}
\newtheorem{theorem}{Theorem}

\newtheorem{fact}{Fact}
\theoremstyle{definition}
\newtheorem{definition}{Definition}

\theoremstyle{remark}

\numberwithin{corollary}{section}
\numberwithin{lemma}{section}
\numberwithin{theorem}{section}
\numberwithin{assumption}{section}
\numberwithin{fact}{section}
\numberwithin{definition}{section}
\numberwithin{example}{section}
\numberwithin{conjecture}{section}
\numberwithin{remark}{section}
\numberwithin{claim}{section}

\newcommand \E {\mathop{\mbox{\ensuremath{\mathbb{E}}}}\nolimits}

\newcommand \CA {{\mathcal{A}}}

\newcommand \defn {\mathrel{\triangleq}}

\newcommand \argmax{\mathop{\rm arg\,max}}

\DeclareMathAlphabet{\mathpzc}{OT1}{pzc}{m}{it}

\newcommand \Laplace {\mathop{\mathpzc{Lap}}\nolimits}

\newcommand \nactions {\ensuremath{K}}

\newcommand \horizon {\ensuremath{T}}

\newcommand \pol {\pi}

\newcommand \arm {a}

\newcommand \pMean {X}%\widetilde{X}}
\newcommand \eMean {Y}%\widehat{X}}
\newcommand \vpMean {\vectorsym{\pMean}}
\newcommand \veMean {\vectorsym{\eMean}}
\newcommand \dpbound {h}
\newcommand \cbound {c}

\newcommand \nat {n_{a,t}} %% number of times arm a played till time t

\newcommand \Mean {\mu}

\newcommand \rt {\vectorsym{r}_t}

\DeclarePairedDelimiter{\abs}{\lvert}{\rvert}
\DeclarePairedDelimiter{\ceil}{\lceil}{\rceil}

\usepackage{tikz}
\usetikzlibrary{automata,topaths,shapes,arrows,fit,decorations.markings}
\tikzstyle{utility}=[diamond,draw=black,draw=blue!50,fill=blue!10,inner sep=0mm, minimum size=8mm]
\tikzstyle{select}=[rectangle,draw=black,draw=blue!50,fill=blue!10,inner sep=0mm, minimum size=6mm]
\tikzstyle{hidden}=[dashed,draw=black]
\tikzstyle{RV}=[circle,draw=black,draw=blue!50,fill=blue!10,inner sep=0mm, minimum size=6mm]

\def\clap#1{\hbox to 0pt{\hss#1\hss}}

\makeatletter
\let\OldStatex\Statex
\renewcommand{\Statex}[1][3]{%
	\setlength\@tempdima{\algorithmicindent}%
	\OldStatex\hskip\dimexpr#1\@tempdima\relax}
\makeatother

\newcommand \DPUCBi {\textsc{DP-UCB-Int}}
\newcommand \DPUCB {\textsc{DP-UCB}}
\newcommand \DPUCBb {\textsc{DP-UCB-Bound}}

\begin{document}

\title{Algorithms for Differentially Private Multi-Armed Bandits}
%\author{Aristide C. Y. Tossou and Christos Dimitrakakis}
\author
{Aristide C. Y. Tossou and Christos Dimitrakakis
	\\\\
	Chalmers University of Technology, Gothenburg, Sweden\\\\
	\{aristide, chrdimi\}@chalmers.se
}

\maketitle

\begin{abstract}%   <- trailing '%' for backward compatibility of .sty file
  We present differentially private algorithms for the stochastic
  Multi-Armed Bandit (MAB) problem. This is a problem for applications
  such as adaptive clinical trials, experiment design, and
  user-targeted advertising where private information is connected to
  individual rewards. Our major contribution is to show that there
  exist $(\epsilon, \delta)$ differentially private variants of Upper
  Confidence Bound algorithms which have optimal regret,
  $O(\epsilon^{-1} + \log T)$. This is a significant improvement over
  previous results, which only achieve poly-log regret
  $O(\epsilon^{-2} \log^{2} T)$, because of our use of a novel
  interval-based mechanism. We also substantially improve the
  bounds of previous family of algorithms which use a continual release
  mechanism. Experiments clearly validate our theoretical bounds. 
  %Our approach is easily extensible to distributed bandit problems.
\end{abstract}

\section{Introduction}

The well-known stochastic $\nactions$-armed bandit
problem~\cite{thompson1933lou,robbins1952some}
involves an agent sequentially
choosing among a set of arms $\CA = \{1, \ldots, \nactions\}$, and
obtaining a sequence of scalar rewards $\{r_t\}$, such that, if the
agent's action at time $t$ is $a_t = i$, then it obtains reward $r_t$
drawn from some distribution $P_{i}$
with expectation $\Mean_i \defn \E(r_t \mid a_t = i)$. The goal
of the decision maker is to draw arms so as to maximize the total
reward $\sum_{t=1}^T r_t$ obtained.

This problem is a model for many applications where there is a need
for trading-off exploration and exploitation. This occurs because we
only see the reward of the arm we pull. An example is clinical trials,
where arms correspond to different treatments or tests, and the goal
can be to maximise the number of cured patients over time while being
uncertain about the effects of treatments.  Other problems, such as
search engine advertisement and movie recommendations can be
formalised similarly \cite{handlingads}.

It has been previously noted~\cite{dponlinelearning,dplearningbanditandfull,mishra2015nearly,dpsmartgrid} that privacy
is an important consideration for many multi-armed bandit applications. Indeed, privacy can be easily violated by observing changes in the prediction of the bandit algorithm. This has been demonstrated for recommender systems such as Amazon by \cite{privacyleakamazon} and for  user-targeted advertising such as Facebook by \cite{privacyleakfacebook}. In both cases, with a moderate amount of side information and by tracking changes in the output of the system, it was possible to learn private information of any targeted user.

Differential privacy (DP) \cite{dwork06dp} provides an answer to this privacy issue by making the output of an algorithm almost insensitive to any single user information. That is, no matter what side information is available to an outside observer, he can not have more information about a user than he already had by observing the outputs released by the algorithm. This goal is achieved by formally bounding the loss in privacy through the use of two parameters $(\epsilon, \delta)$ as shown in Definition \ref{def:dp}.
%In this paper, we formalise privacy through
%the notion of differential privacy (DP), which is both theoretically sound
%and has found applications in many areas~\cite{dpbook}.  

For bandit problems, differential privacy implies that the actions taken by the
bandit algorithm do not reveal information about the sequence of
rewards obtained. In the context of clinical trials and diagnostic
tests, it guarantees that even an adversary with arbitrary side
information, such as the identity of each patient, cannot learn
anything from the output of the learning algorithm about patient
history, condition, or test results.

\subsection{Related Work}
\label{sec:related-work}
Differential privacy (DP) was introduced by \cite{noisedpdwork}; a
good overview is given in \cite{dpbook}. While initially the focus in
DP was static databases, interest in its relation to online learning
problems has increased recently. In the full information setting,
\cite{dponlinelearning} obtained differentially private algorithms
with near-optimal bounds.  In the bandit setting,
\cite{dplearningbanditandfull} were the first to present a
differentially private algorithm, for the adversarial case, while
\cite{dpsmartgrid} present an application to smart grids in this
setting.  Then,
\cite{mishra2015nearly} provided a differentially private algorithm
for the stochastic bandit problem.  Their algorithms are based on two
non private stochastic bandit algorithms: Upper Confidence Bound
(UCB,~\cite{finitetimemab}) and Thompson
sampling~\cite{thompson1933lou}.  Their results are sub-optimal:
although simple index-based algorithms achieving $O(\log T)$ regret
exist~\cite{burnetas1996optimal,finitetimemab}, these differentially
private algorithms additional poly-log terms in time $T$, as well
further linear terms in the number of arms compared to the non-private
optimal regret $O(\log T)$.

We provide a significantly different and improved UCB-style algorithm whose regret only adds a constant, privacy-dependent term
to the optimal.  We also improve upon previous algorithms by relaxing
the need to know the horizon $T$ ahead of time, and as a result we
obtain a uniform bound. Finally, we also obtain significantly improved
bounds for a variant of the original algorithm
of~\cite{mishra2015nearly}, by using a different proof technique and
confidence intervals.  Let's note that similarly to their result, we
only make distributional assumptions on the data for the regret
analysis. To ensure privacy, our algorithms do not make any assumption
on the data. We summarize our contributions in the next section.

\subsection{Our Contributions}

\begin{itemize}
	\item We present a novel differentially private algorithm (\DPUCBi{}) in the stochastic bandit setting that is almost optimal and only add an \emph{additive constant} term (depending on the privacy parameter) to the optimal non private version. Previous  algorithms had in large \emph{multiplicative factors} to the optimal.
	
	\item We also provide an incremental but important improvement to the regret of existing differentially private algorithm in the stochastic bandit using the same family of algorithms as previously presented in the literature. This is done by using a simpler confidence bound and a more sophisticated proof technique. These bounds are achieved by \DPUCBb{} and \DPUCB{} algorithms.

	\item We present the first set of differentially private algorithm in the bandit setting which are unbounded and do not require the knowledge of the horizon $T$. Furthermore, all our regret analysis holds for any time step $t$.
	
\end{itemize}

\section{Preliminaries}

\subsection{Multi-Armed Bandit}
The well-known stochastic $\nactions$-armed bandit problem~\cite{thompson1933lou,lai1985asymptotically,finitetimemab} involves an agent sequentially choosing among a set of $\nactions$ arms $\CA = \{1, \ldots, \nactions\}$. At each time step $t$, the player selects an action $a_t = i \in \CA$  and obtains a reward $r_t \in [0,1]$.  The  reward $r_t$ is drawn from some fixed but unknown distribution $P_{i}$ such that $\E(r_t \mid a_t) = \Mean_i$. The goal of the decision maker is to draw arms so as to maximize the  total reward obtained after $T$ interactions. An equivalent notion is to minimize the total regret against an agent who knew the arm with the maximum expectation before the game starts and always plays it. This is defined by:
\begin{equation}
\mathcal{R} \defn
%\E^\pol_\vparam R_\horizon = 
\horizon
\Mean_{*}
-
\E^\pol
\sum_{t=1}^\horizon
\rt
\label{eq:expected-regret}.
\end{equation}
where $\Mean_{*} \defn \max_{a \in \CA} \Mean_{a}$ is the mean reward of the optimal arm 
and  $\pol(a_t \mid a_{1:t-1}, r_{1:t-1})$ is the policy of the decision maker, defining a probability distribution on the next actions $a_t$ given the history of previous actions $a_{1:t-1} = a_1, \ldots, a_{t-1}$ and rewards $r_{1:t-1}  = r_1, \ldots r_{t-1}$.
Our goal is to bound the regret uniformly over $T$.

\subsection{Differential Privacy}
Differential privacy was originally proposed by~\cite{dwork06dp},
as a way to formalise the amount of information about the \emph{input}
of an algorithm, that is leaked to an adversary observing its
\emph{output}, no matter what the adversary's side information is. In
the context of our setup, the algorithm's input is the sequence of
rewards, and its output the actions. Consequently, we use the following definition of differentially private bandit algorithms.
\begin{definition}[($\epsilon, \delta)$-differentially private bandit algorithm]
  A bandit algorithm $\pi$ is $(\epsilon, \delta)$-differentially
  private if for all sequences $r_{1:t-1}$ and $r'_{1:t-1}$ that
  differs in at most one time step, we have for all $S
  \subseteq \CA$:
  %	\[
  	%\pi(a_t \in S \mid a_{1:t-1}, r_{1:t-1})
  	%\leq 
  	%\pi(a_t \in S \mid a_{1:t-1}, r'_{1:t-1})
  	%e^\epsilon + \delta
\[
\pol(a_t \in S \mid a_{1:t-1}, r_{1:t-1})
\leq 
\pol(a_t \in S \mid a_{1:t-1}, r'_{1:t-1})
e^\epsilon + \delta
\label{eq:dp-bandit}
\]
 where $\CA$ is the set of actions. When $\delta = 0$, the algorithm is said to be \emph{$\epsilon$-differential private}.
 \label{def:dp}
\end{definition}
Intuitively, this means that changing any reward $r_t$ for a given arm, will not change too much the best arm released at time $t$ or later on. If each $r_t$ is a private information or a point associated to a single individual, then the definition aboves means that the presence or absence of that individual will not affect too much the output of the algorithm. Hence, the algorithm will not reveal any extra information about this individual leading to a privacy protection.
The privacy parameters $(\epsilon, \delta)$ determines the extent to which an individual entry affects the output; lower values of $(\epsilon, \delta)$ imply higher levels of privacy.

A natural way to obtain privacy is to add a noise such as Laplace noise ($\Laplace{}$) to the output of the algorithm. The main challenge is how to get the maximum privacy while adding a minimum amount of noise as possible. This leads to a trade off between privacy and utility. In our paper, we demonstrated how to optimally trade-off this two notions.

\begin{comment}
[Talk about how adding laplace mechanism can lead to DP.]
\begin{fact}[Corollary 2.9 in~\cite{chan2010private}]
If $\lambda_i \sim \Laplace(b_i)$ and $Y = \sum_i \lambda_i$ then
\begin{equation}
\Pr(|Y| \geq \|b\|_2 \ln \frac{1}{\delta}) \leq \delta.
\label{eq:laplace}
\end{equation}
\end{fact}
\end{comment}

 \subsection{Hybrid Mechanism}
 
 The hybrid mechanism is an online algorithm used to continually
 release the sum of some statistics while preserving differential
 privacy. More formally, there is a stream $\sigma_t = {r_1, r_2
   \ldots r_t}$ of statistics with $r_i$ in $[0, 1]$. At each time
 step $t$ a new statistic $r_t$ is given. The goal is to output the
 partial sum ($y_t = \sum_{i=1}^{t} r_i$) of the statistics from time
 step 1 to $t$ without compromising privacy of the statistics. In other
 words, we wish to find a randomised mechanism $M(y_t \mid \sigma_t,
 y_{1:t-1})$ that is $(\epsilon, \delta)$-differential private.

The hybrid mechanism solves this problem by combining the Logarithm and Binary Noisy Sum mechanisms.
Whenever $t = 2^k$ for some integer $k$, it uses the Logarithm mechanism to release a noisy sum by adding Laplace noise of scale $\epsilon^{-1}$. It then builds a binary tree $B$, which is used to release noisy sums until $t = 2^{k+1}$ via the Binary mechanism.
 This uses the leaf nodes of $B$ to store the inputs $r_i$, while all other nodes store partial sums, with the root containing the sum from $2^k$ to $2^{k+1} -1$. Since the tree depth is logarithmic, there is only a logarithmic amount of noise added for any given sum, 
more specifically Laplace noise of scale $\frac{\log t}{\epsilon}$ and mean $0$ which is denoted by $\Laplace(\frac{\log t}{\epsilon})$.

\cite{chan2010private} proves that the hybrid mechanism is $\epsilon$-differential private for any $n$ where $n$ is the number of statistics seen so far. They also show that with  probability at least $1-\gamma$, the error in the released sum is upper bounded by $\frac{1}{\epsilon} \log(\frac{1}{\gamma}) \log^{1.5} n$. 
In this paper, we derived and used a tighter bound for this same mechanism (see Appendix in Supplementary Material) which is:
\begin{equation}
\frac{\sqrt{8}}{\epsilon} \log(\frac{4}{\gamma}) \log n + \frac{\sqrt{8}}{\epsilon} \log(\frac{4}{\gamma})
\label{hybrid-mechanism-new-bound}
\end{equation}

\section{Private Stochastic Multi-Armed Bandits}
\label{sec:priv-distr-multi}

%\subsection{General Description of the Algorithms}
We describe here the general technique used by our algorithms to obtain differential privacy.
Our algorithms are based on the non-private UCB algorithm by \cite{finitetimemab}. At each time step, UCB based its action according to an optimistic estimate of the expected reward of each arm. This estimate is the sum of the empirical mean and an upper bound confidence  equal to $\sqrt{\frac{2\log t}{\nat}}$ where $t$ is the time step and $\nat$ the number of times arm $\arm$ has been played till time $t$.  We can observe that the only quantity using the value of the reward is the  empirical mean. To achieve  differential privacy, it is enough to make the player based its action on  \emph{differentially private} empirical
means for each arm. This is so, because, once the mean of each arm is computed, the action which will be played is a deterministic function of the means. In particular, we can see the differentially private
mechanism as a black box, which keeps track of the vector of
non-private empirical means $\veMean$ for the player, and outputs a vector
of private empirical means $\vpMean$. This is then used by the player to
select an action, as shown in
%the graphical model of
Figure~\ref{fig:graphical-model}.

We provide three different algorithms that use different techniques to privately compute the mean and calculate the index of each arm. The first, \DPUCBb{}, employs the Hybrid mechanism to compute a private mean and then adds a suitable term to the confidence bound to take into account the additional uncertainty due to privacy. The second, \DPUCB{} employs the same mechanism, but in such a way so as all arms have the same privacy-induced uncertainty; consequently the algorithm then uses the same index as standard UCB. The final one, employs a mechanism that only releases a new mean once at the beginning of each interval. This allows us to obtain the optimal regret rate.

% Christos: What about minimax rates from Duchi and the Lai and Robbins analysis? We should be able to get a lower bound for the regret for any DP bandit algorithm.

\begin{figure}[ht]
  \centering
  \begin{tikzpicture}
    \node[select] at (0,1.5) (at) {$a_t$};
    \node[utility,hidden] at (2,1.5) (rt) {$r_t$};
    \node[RV,hidden] at (0,0) (m) {$P_i$};
    \node[RV,hidden] at (2,0) (em) {$\veMean$};
    \node[RV] at (4,0) (pm) {$\vpMean$};
    \draw[->] (at) -- (rt);
    \draw[->] (m) -- (rt);
    \draw[->] (rt) -- (em);
    \draw[->] (em) -- (pm);
    \node (slab) [draw=red, fit= (at) (rt), inner sep=1em] {};
    \node [yshift=-2ex, blue] at (slab.north) {$t \in [T]$};
  \end{tikzpicture}
  \caption{Graphical model for the empirical and private means. $a_t$ is the action of the agent, while $r_t$ is the reward obtained, which is drawn from the bandit distribution $P_i$. The vector of empirical means $\veMean$ is then made into a private vector $\vpMean$ which the agent uses to select actions. The rewards are essentially hidden from the agent by the DP mechanism.} %differentially private mechanism.}
  \label{fig:graphical-model}
\end{figure}
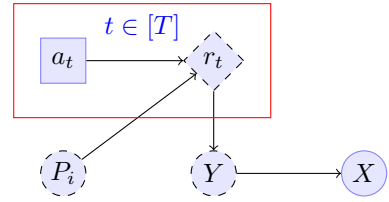

\subsection{The \DPUCBb{} Algorithm}
\label{sec:private-hybrid-mab}

\begin{algorithm}                      
  \caption{\DPUCBb{}}          
  \label{alg1}                           
  \begin{algorithmic}
  	 \State \textbf{Input} $\epsilon$, the differential privacy parameter.
    \State Instantiate  $\nactions$ Hybrid Mechanisms \cite{chan2010private}; one for each arm  $a$.
    
    \For{$t \gets 1$ to $T$}
    \If {$t \leq \nactions$}
    \State play  arm $a = t$ and observe the reward $r_t$
    \State Insert $r_t$ to the hybrid mechanism $a$

    \Else

	\ForAll{arm $a$}
    \State $s_a(t) \gets$ total sum computed using the hybrid 
    \Statex[4] $\qquad$ mechanism $a$
    
    \If{$\nat$ is a power of 2}
    
     \State $\nu_a \gets \frac{\sqrt{8}}{\epsilon} \log(4t^4)$
     
    \Else
    
    \State $\nu_a \gets \frac{\sqrt{8}}{\epsilon} \log(4t^4) \log \nat + \frac{\sqrt{8}}{\epsilon} \log(4t^4)$
    
    \EndIf
    
    \EndFor

    \State Pull arm 
    $a_t = \argmax_a \frac{s_a(t)}{\nat}  +  \sqrt{\frac{2\log{t}}{\nat}} + \frac{\nu_a}{\nat}$
    \State Observe the reward $r_{t}$
	\State Insert $r_{t}$ to the hybrid mechanism for arm $a_t$
    \EndIf
    \EndFor

  \end{algorithmic}
\end{algorithm}

In Algorithm \ref{alg1}, we compute the sum of the rewards of each arm using the hybrid mechanism \cite{chan2010private}. However, the number and the variance of Laplace noise added by the hybrid mechanism increases as we keep pulling an arm. This means that the sum of each arm get added different amount of noise bigger than the original confidence bound used by UCB. This makes it difficult to identify the best arms. To solve this issue, we add a tight upper bound defined in equation \eqref{hybrid-mechanism-new-bound} on the noise added by the hybrid mechanism.

Theorem \ref{theo:bound}  validates this choice by showing that only $O( \epsilon^{-1} \log \log t)$ factors are added to the optimal non private regret. In theorem \ref{theo:privacy_algo1}, we demonstrate that Algorithm \ref{alg1} is indeed $\epsilon$-differential private.

\begin{comment}
	 We will now show that Algorithm \ref{alg1} is both private and leads to poly-log regret. The first statement follows from the fact that each agent is using a differentially private statistic to select actions. The second follows from concentration inequalities that bound the error in each round of the algorithm and the fact that the confidence bound we used results in a poly-log number of erroneous arm pulls.
	 
	 Let us first discuss the privacy of the algorithm.
\end{comment}

\begin{theorem}
  Algorithm \ref{alg1} is $\epsilon$-differential private  after any number of $t$ of plays.
  \label{theo:privacy_algo1}
\end{theorem}
\begin{proof}
This follows directly from the fact that the hybrid mechanism is $\epsilon$-DP after any number $t$ of  plays and a single flip of one reward in the sequence of rewards only affect one mechanism. Furthermore,  the whole algorithm is a random mapping from the output of the hybrid mechanism to the action taken and using Proposition 2.1 of \cite{dpbook} completes the proof.
\end{proof}

Theorem \ref{theo:bound} gives the regret for algorithm \ref{alg1}. 
While here we give only a sketch proof, the complete derivation can be found in the supplementary material \cite{dpsups}.
\begin{theorem}
  If Algorithm \ref{alg1} is run with $\nactions$ arms having arbitrary reward distributions,
  %$P_1$, \dots, $P_\nactions$ with support in $[0, 1]$, 
  then, its expected regret $\mathcal{R}$ after any number $t$ of plays is bounded by:
  \begin{align}
      \mathcal{R} &\leq   \sum_{a:\mu_a < \mu_{*}}  
      \max\left(B  \left(\ln B  +7\right), \frac{8}{\lambda_0^2\Delta_a} \log{t}\right) \notag \\ &+\sum_{a:\mu_a < \mu_{*}} \left(\Delta_{a} + \frac{2\pi^2\Delta_{a}}{3}\right)
    \end{align}
    \[B = \frac{\sqrt{8}}{\epsilon(1-\lambda_0)}  \cdot \ln(4t^{4})\]
    for any  $\lambda_0$ such that $0 < \lambda_0 < 1$
   where $\mu_1$, \dots, $\mu_\nactions$ are the expected values of $P_1$, \dots, $P_\nactions$ and $\Delta_a = \mu_{*} - \mu_a$.
  \label{theo:bound}
\end{theorem}
\begin{proof}[Proof Sketch]
	We used the bound on the hybrid mechanism defined in equation \ref{hybrid-mechanism-new-bound} together with the union and  Chernoff-Hoeffding bounds. We then select  the error probability $\gamma$ at each step to be $t^{-4}$. This leads to a transcendental inequality solved using the Lambert W function and approximated using section $3.1$ of \cite{Barry200095}.
\end{proof}

\subsection{The \DPUCB{} Algorithm}
\label{sec:priv-hybr-algor}
The key observation used in Algorithm \ref{alg2} is that if at each time step we insert a reward to all hybrid mechanisms, then the scale of the noise will be the same. This means that there is no need anymore to compensate an additional bound. More precisely, every time we play an arm $a_t = a$ and receive the reward $r_t$,  we not only add it to the hybrid mechanism corresponding to arm $a$ but we also add a reward of $0$ to the hybrid mechanism of all other arms. As these calculate a sum, it doesn't affect subsequent calculations.

% Christos: I think you can use exactly the same technique for the basic partial monitoring algorithm; since there you add a estimated reward anyway, it'll be trivial to do the same here.

Theorem \ref{theo:bound2} shows the validity of this approach by demonstrating a regret bound with only an additional factor of $O(\epsilon^{-2} \log^2 \log t)$ to the optimal non private regret. 
%On the other hand, the dependency on $\epsilon$ is greater.

\begin{comment}
	Instead of adding the hybrid mechanism bound to the UCB bound, we can make sure that the variance of the noise added for each arm is the same. To do that, when we play an arm $a$ and insert the return reward to the hybrid mechanism, we also insert a rewards of 0 to all other arms. This idea is sketched in Algorithm \ref{alg2}.
\end{comment}

\begin{algorithm}                      
  \caption{\DPUCB{}}
  \label{alg2}                           
  \begin{algorithmic}
    \State Run Algorithm \ref{alg1} and set $\nu_a$ to $0$
    \State When arm $a_t = a$ is played,
    insert $0$ to all hybrid mechanisms corresponding to
    arm $a' \ne a$ (Do not increase $n_{a',t}$)
  \end{algorithmic}
\end{algorithm}

\begin{comment}
Algorithm \ref{alg2} also enjoys poly-log regret as show in theorem \ref{theo:bound2}.
\end{comment}
\begin{theorem}
  If Algorithm \ref{alg2} is run with $\nactions$ arms having arbitrary reward distributions,
  %$P_1$, \dots, $P_\nactions$ with support in $[0, 1]$, 
  then, its expected regret $\mathcal{R}$ after any number $t$ of plays is bounded by:
  \begin{align*}
    \mathcal{R} &\leq   \sum_{a:\mu_a < \mu_{*}}  \Delta_{a} \left[
   \max\left(C^2  \left(\ln C  +7\right)^2, \frac{8}{\Delta_a^2} \log{t}\right) \right]\notag \\
   &+\sum_{a:\mu_a < \mu_{*}} \left(\Delta_{a} + 4\Delta_{a} \zeta(1.5)\right) 
  \end{align*}
\[ C = \frac{56(2+\sqrt{3.5}) \sqrt{\log t} }{\epsilon} \]
  where $\zeta$ denotes the Riemann zeta function. 
  
  \label{theo:bound2}
\end{theorem}
\begin{proof}[Proof Sketch]
 The proof is similar to the one for Theorem \ref{theo:bound}, but we have to choose the error probability to be $t^{-3.5}$.
\end{proof}

\subsection{The \DPUCBi{} Algorithm}
\label{sec:priv-interv-based}
Both Algorithms \ref{alg1} and \ref{alg2} enjoy a logarithmic regret with only a small additional factor in the time step $t$ to the optimal non-private regret. However, this includes a multiplicative factor of $\epsilon^{-1}$ and  $\epsilon^{-2}$ respectively. Consequently, increasing privacy scales the total regret proportionally.  A natural question is whether or not it is possible to get a differentially private algorithm with only an \emph{additive} constant to the optimal regret. Algorithm \ref{alt2} answers positively to this question by using novel tricks to achieve differential privacy. 
Looking at regret analysis of Algorithms \ref{alg1} and \ref{alg2}, we observe that by adding noise proportional to $\epsilon$, we will get a multiplicative factor to the optimal. In other words, to remove this factor, the noise should not depend on $\epsilon$.  But how can we get  $\epsilon$-DP in this case?

\begin{comment}
	The goal of interval-based algorithm is to remove the dependency of $\epsilon$ in the poly-log term. To be able to do that, we have the noise on the mean should not depend on $\epsilon$. However, this leads to the obvious question of how we can achieve DP in this case.
\end{comment}

Note that if we compute and use the mean at each time step with an $\epsilon'_{\nat}$-DP algorithm, then after time step $t$, our overall privacy is roughly the sum $\cal{E'}$ of all $\epsilon'_{\nat}$. We then change the algorithm so that it only uses a released mean once every $\frac{1}{\epsilon}$ times, making privacy $\epsilon \cal{E'}$. In any case,  $\epsilon'_{\nat}$ needs to decrease, at least as $\nat^{-1}$, for the sum to be bounded by $\log \nat$. However, $\epsilon'_{\nat}$ should also be big enough such that the noise added keeps the UCB confidence interval used at the same order,  otherwise, the regret will be higher.

A natural choice for $\epsilon'_{\nat}$ is a p-series. Indeed, by making $\epsilon'_{\nat}$ to be of the form $\frac{1}{\nat^{v/2}}$, where $\nat$ is the number of times action $a$ has been played until time $t$, its sum will converge to the Riemann zeta function when $v$ is appropriately chosen. This choice of  $\epsilon'_{\nat}$ leads to the addition of a Laplace noise of scale $\frac{1}{\nat^{1-v/2}}$ to the mean (See Lemma \ref{lemma:interval_dp_step}). Now our trade-off issue between high privacy and low regret is just reduced into choosing a correct value for $v$. Indeed, we can pick $v > 2$, for the privacy to converge; but the noise added at each time step will be increasing and greater than the UCB bound; which is not desirable.
To overcome this issue, we used the more sophisticated $k$-fold adaptive composition theorem (III-3 in \cite{dworkrv10}). Roughly speaking, this theorem shows that  our overall privacy after releasing the mean a number of times depends on the sum of the \emph{square} of each individual privacy parameter $\epsilon'_{\nat}$. So, $v > 1 $ is enough for convergence and with $v \leq 1.5$, the noise added will be decreasing and will eventually become lower than the UCB bound.

In summary, we just need to \emph{lazily update the mean} of each arm every $\frac{1}{\epsilon}$ times. However, we show that the interval of release is much better than $\frac{1}{\epsilon}$ and follows a series $f$ as defined by Lemma (B.1) in the supplements \cite{dpsups}. Algorithm $\ref{alt2}$ summarizes the idea developed in this section.

The next lemma establishes the privacy $\epsilon'$ each time a new mean is released for a given arm $a$.
\begin{algorithm}[htb]                      
  \caption{\DPUCBi{} ($\epsilon$, $v$, $\nactions$, $\mathcal{A}$)}
  \label{alt2}                           
  \begin{algorithmic}
    \State \textbf{Input} $\epsilon \in (0, 1]$ $v \in (1, 1.5]$; privacy rate.
	\State $\nactions$ is the number of arms and $\mathcal{A}$ the set of all arms.
	\State	$f \gets \ceil{\frac{1}{\epsilon}}$;  $\hat{x} \gets 0$
	
	\State (For simplicity, we take the interval $f$ to be $\ceil{\frac{1}{\epsilon}}$ here)
    
    \For{$t \gets 1$ to $T$}
    \If {$t \leq \nactions f$}
    \State play  arm $a = (t-1)\mod{\nactions} + 1$ and observe $r_t$

    \Else

    \ForAll{$a \in \mathcal{A}$}

   	\If {$\nat \mod{f} = 0$}
   		\State $\hat{x}_a \gets \frac{s_a}{\nat} + \Laplace(0, \frac{1}{\nat^{1-v/2}}) + \sqrt{\frac{2\log{t}}{\nat}}$
   		%\State $\text{confidence}_a \gets \sqrt{\frac{2\log{t}}{\nat}}$
   	\EndIf

   	%\State $\cbound_a \gets \hat{x}_a + \mathcal{N}(0, \text{confidence}_a^2)$
   	%\State $\cbound_a \gets \hat{x}_a$

    \EndFor
    
    \State Pull arm $a_t = \argmax_a \hat{x}_a$ and observe 
    $r_t$
    \State Update sum $s_a \gets s_a + r_t$.
    \EndIf

    \EndFor

  \end{algorithmic}
\end{algorithm}
\begin{lemma}
The mean $\hat{x}_a$ computed by Algorithm \ref{alt2} for a given arm $a$ at each interval is $\nat^{-v/2}$-differential private with respect to the reward sequence observed by that arm.
\label{lemma:interval_dp_step}
\end{lemma}

\begin{proof}{Sketch}
	This follows directly from the fact that we add Laplace noise of scale $\nat^{v/2-1}$.
\end{proof}
The next theorem establishes the overall privacy after having played for $t$ time steps.
\begin{theorem}
After playing for any $t$ time steps, Algorithm \ref{alt2} is $(\epsilon', \delta')$-differential private with
\[ \epsilon' \leq   \min\left( \sum_{n=1}^{t} \frac{\epsilon}{\sqrt{n^{v}}} ,\epsilon \sum_{n=1}^{t} \frac{e^{\frac{1}{\sqrt{n^{v}}}}-1}{\sqrt{t^{v}}}  + \sqrt{\epsilon\sum_{n=1}^{t} \frac{2\ln\frac{1}{\delta'}}{n^{v}} } \right) \]
for any $\delta' \in (0, 1]$, $\epsilon \in (0, 1]$
\label{theo:interval_simple_privacy}
\end{theorem}

\begin{proof}[Proof Sketch]
We begin by using similar observations as in Theorem \ref{theo:privacy_algo1}. Then, we compute the privacy of the mean of an arm using the $k$-fold adaptive composition theorem in \cite{dworkrv10} (see the supplements \cite{dpsups}).
	
\end{proof}
%In practice, we would like to know a closed form for the differential privacy parameter. 
The next corollary gives a nicer closed form for the privacy parameter which is needed in practice.
\begin{corollary}
After playing for $t$ time steps, Algorithm \ref{alt2} is $(\epsilon', \delta')$-differential private with
\[\epsilon' \leq  \min \left(  \epsilon \frac{t^{1-v/2} -v/2}{1-v/2}, 2 \epsilon\zeta(v) + \sqrt{ 2 \epsilon\zeta(v) \ln(1/\delta')} \right)\]
with $\zeta$ the Riemann Zeta Function for any $\delta' \in (0, 1]$, $\epsilon \in (0, 1]$, $v \in (1, 1.5]$.
\label{cor:alt2:dp}
\end{corollary}
\begin{proof}[Proof Sketch]
	We upper bounded the first term in theorem \ref{theo:interval_simple_privacy} by the integral test, then for the second term we used $e^{x} \leq 1 + 2x$ for all $x \in [0, 1]$ to conclude the proof.
\end{proof}
The following corollary gives the parameter $\epsilon$ with which one should run Algorithm \ref{alt2} to achieve a given $\epsilon'$ privacy.
\begin{corollary}
If you run Algorithm \ref{alt2} with parameter
$\epsilon = \left(\sqrt{\frac{\log\frac{1}{\delta'} + 4\epsilon'}{8\zeta(v)}} - \sqrt{\frac{\log\frac{1}{\delta'}}{8 \zeta(v)}} \right)^2$ for any $\delta' \in (0, 1]$, $\epsilon' \in (0, 1]$, $v \in (1, 1.5]$, 
you will be at least $(\epsilon', \delta')$-differential private.
\label{alt2_cor}
\end{corollary}

\begin{proof}
	The proof is obtained by inverting the term using the Riemann zeta function in corollary  \ref{cor:alt2:dp}.
\end{proof}

Finally, we present the regret of Algorithm \ref{alt2} in theorem \ref{theo:regret_interval} . A simple observation shows us that it has the same regret as the non private UCB with just an additive constant.
\begin{theorem}
  If Algorithm \ref{alt2} is run with $\nactions$ arms having arbitrary reward distributions,
  %$P_1$, \dots, $P_\nactions$ with support in $[0, 1]$, 
  then, its expected regret $\mathcal{R}$ after any number $t$ of plays is bounded by:
  \begin{equation*}
    \mathcal{R} \leq   \sum_{a:\mu_a < \mu_{*}}  \Delta_{a} \left[
    f_0 + \frac{8}{\Delta_a^2} \log t + 1+ 4 \zeta(1.5) \right]
  \end{equation*}
  where $f_0 \leq \frac{1}{\epsilon}$. More precisely, $f_0$ is the first value of the series $f$ defined in Lemma B.1 in the supplements \cite{dpsups}.
  %where $C(\epsilon, v)$ is a function of $\epsilon$ and $v$ such that $C(\epsilon, v) \leq \frac{1}{\epsilon}$.
  %More precisely, $C(\epsilon, v)$ is lower or equal to the minimum number $n$ such that $\sum_{t=1}^{n} \frac{1}{t^v} \geq \frac{1}{\epsilon n^v}$

  \label{theo:regret_interval}
\end{theorem}
\begin{proof}[Sketch]
  This is proven using a Laplace concentration inequality to bound the estimate of the mean then we selected the error probability to be $t^{-3.5}$.
\end{proof}

	\begin{figure*}[htb]
	\centering
	\subfloat[Regret for $\epsilon = 1$, 2 arms]{
		\scalebox{0.4}{\input{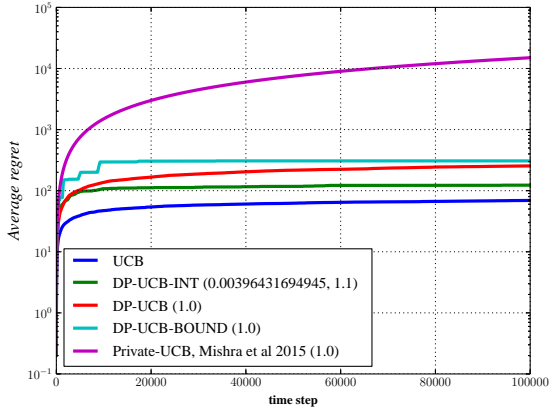}}
	}
	\subfloat[Regret for $\epsilon = 0.1$, 2 arms]{
		\scalebox{0.4}{\input{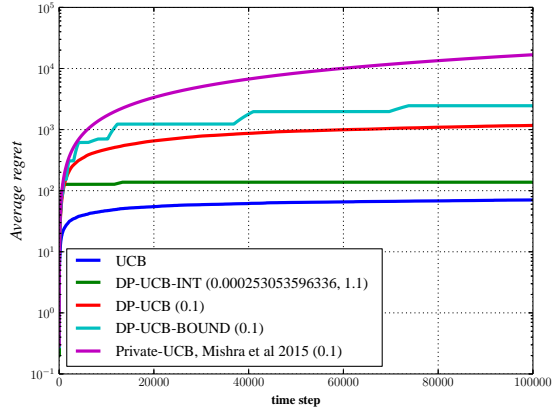}}
	}
	\\
	\subfloat[Regret for $\epsilon = 1$, 10 arms]{
	\scalebox{0.4}{\input{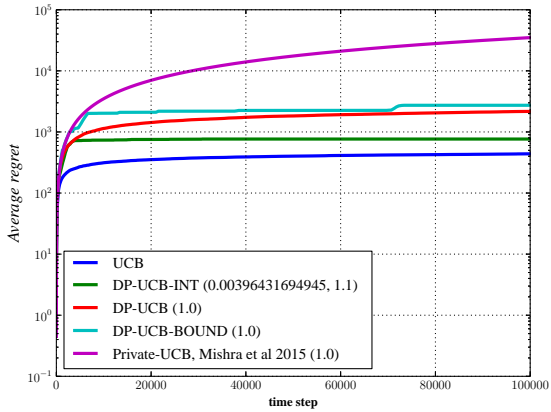}}
	}
	\subfloat[Regret for $\epsilon = 0.1$, 10 arms]{
	\scalebox{0.4}{\input{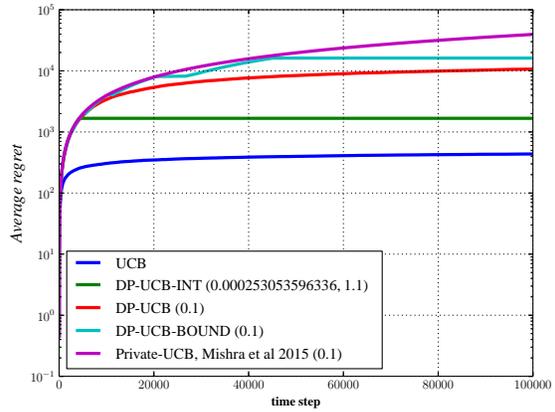}}
	}	
	%\\
	%\subfloat[Regret for $\epsilon = 0.01$]{
	%  \scalebox{0.4}{\input{figures/scen_2_regret_cumul.pgf}}
	%}
	\caption{Experimental results with 100 runs, 2 or 10 arms with rewards: $\{0.9, 0.6\}$ or $\{0.1 \ldots 0.2, 0.55, 0.1 \ldots \}$ .}
	\label{fig:experiments}
	\end{figure*}

\section{Experiments}
\label{sec:experiments}
We perform experiments using arms with rewards drawn from independent Bernoulli distribution.
The plot, in logarithmic scale, shows the regret of the algorithms over $100,000$ time steps averaged over 100 runs. 
We targeted 2 different $\epsilon$ privacy levels : 0.1 and 1. For \DPUCBi{}, we pick $\epsilon$ such that the overall privacy is $(\epsilon', \delta')$-DP with $\epsilon$ as defined in corollary \ref{alt2_cor} and $\delta' = e^{-10}$, $\epsilon' \in \{0.1, 1\}, $ $v = 1.1$. We put in parenthesis the \emph{input} privacy of each algorithm.

We compared against the non private UCB algorithm and the algorithm presented in \cite{mishra2015nearly} (\emph{Private-UCB} ) with a failure probability chosen to be $t^{-4}$.

We perform two scenarios. Firstly we used two arms: one with expectation 0.9 and the other 0.6. The second scenario is a more challenging one with 10 arms having all an expectation of 0.1 except two with 0.55 and 0.2.

As expected, the performance of \DPUCBi{} is significantly better than all other private algorithms. More importantly, the gap between the regret of  \DPUCBi{} and the non private UCB does not increase with time confirming the theoretical regret.
We can notice that \DPUCB{} is better than \DPUCBb{} for small time steps. However, as the time step increases \DPUCBb{} outperforms \DPUCB{}  and eventually catches its regret. The reason for that is: \DPUCB{} spends less time to distinguish between arms with close rewards due to the fact that  the additional factor  in its regret depends on $\Delta_a = \mu_* - \mu_a$ which is not the case for \DPUCB{}. Private-UCB performs worse than all other algorithms which is not surprising.

Moreover, we noticed that the difference between the best regret (after 100 runs) and worst regret is very consistent for all ours algorithms and the non private UCB (it is under 664.5 for the 2 arms scenario). However, this gap reaches $30, 000$ for Private-UCB. This means that our algorithms are able to correctly trade-off between exploration and exploitation which is not the case for Private-UCB.

\iffalse
\begin{algorithm}
  \caption{Hybrid DMAB}          
  \label{alg3}                           
  \begin{algorithmic}
  
  \State Instantiate Algorithm \ref{alg1} as $A_1$ and Algorithm \ref{alg2} as $A_2$
  \State At each time step, compute the regret $R_1$ and $R_2$ for $A_1$ and $A_2$ respectively.
  \State Use the empirical mean to estimate the $\Delta$ or 1 when one empirical mean is not available yet.
  \If {$R_1 < R_2$}
    \State Execute and play action $a$ according to $A_1$. Observe the given reward $r_t$.
    \State Run $A_2$ without taking any action. Instead, assume that the action taken is $a$ and the given reward is $r_t$  
    \Else 
    \State Exchange $A_1$ by $A_2$
  \EndIf
  \end{algorithmic}
\end{algorithm}
\fi
\begin{comment}
	We repeat the experiments for 3 players, shown in Figure~\ref{fig:experiments}(c,d), and there we see that the relationships between the algorithms remain the same.
\end{comment}

\section{Conclusion and Future Work}
\label{sec:concl-future-work}
In this paper, we have proposed and analysed differentially private algorithms for the stochastic multi-armed bandit problem, significantly improving upon the state of the art.
The first two, (\DPUCB{} and \DPUCBb{}) are variants of an existing private UCB algorithm~\cite{mishra2015nearly}, while the third one uses an interval-based mechanism.

Those first two algorithms are only within a factor of $O(\epsilon^{-1} \log \log t)$ and $O(\epsilon^{-2} \log^2 \log t)$ to the non-private algorithm. The last algorithm, \DPUCBi{}, efficiently trades off the privacy level and the regret and is able to achieve the same regret as the non-private algorithm up to an additional \emph{additive} constant. This has been achieved by using two key tricks: updating the mean of each arm lazily with a frequency proportional to the privacy $\epsilon^{-1}$ and adding a noise independent of $\epsilon$. Intuitively, the algorithm achieves better privacy without increasing regret, because its output is less dependent on individual reward.

Perhaps it is possible to improve our bounds further if we are willing to settle for asymptotically low regret~\cite{cowan2015asymptotic}.
A natural future work is to study if we can use similar methods for other mechanisms such as Thompson sampling (known to be differentially private~\cite{dimitrakakis:spbi:alt}) instead of UCB. Another question is whether a similar analysis can be performed for \emph{adversarial} bandits.

We would also like to connect more to applications by two extensions of our algorithms. The first natural extension is to consider some side-information. In the drug testing example, this could include some information about the drug, the test performed and the user examined or treated. The second extension would relate to generalising the notion of neighbouring databases to take into account the fact that multiple observations in the sequence (say $m$) can be associated with a single individual. Our algorithms can be easily extended to deal with this setting (by re-scaling the privacy parameter to $\frac{\epsilon}{m}$). However, in practice, $m$ could be quite large and it will be an interesting future work to check if we could get sub linearity in the parameter $m$ under certain conditions.

\appendix
\bibliographystyle{aaai}
\bibliography{bibliography}
\section{Collected proofs}

\subsection{Proof of Theorem \ref{theo:bound}}

$a$ will be used to indicate the index of an arm. $\epsilon$ is the differential privacy parameter.  $t$ is used to denote the time step.

From Lemma \eqref{lem:chan-improved}, we know that the error between the empirical and private mean is bounded as $\abs{\eMean - \pMean} \leq \dpbound_n$ with probability at least $1-\gamma$, $\pMean$ is the empirical mean returned by the private mechanism,
$\eMean$ the true empirical mean, $\dpbound_n$ the error due to the differentially private mechanism. It is defined as: $\dpbound_n = \frac{1}{\epsilon} \cdot   \sqrt{8}(\log{n}) \cdot \ln{\frac{4}{\gamma}} \cdot \frac{1}{n} +  \frac{1}{\epsilon} \cdot   \sqrt{8} \cdot \ln{\frac{4}{\gamma}} \cdot \frac{1}{n}$.
We can rewrite this bound into equations \ref{eq:privateboundupper} and \ref{eq:privateboundlower}.
\begin{align}
  \Pr(\pMean \geq \eMean + h_{n}) &\leq \gamma
  \label{eq:privateboundupper}
  \\
  \Pr(\pMean \leq \eMean - h_{n}) &\leq \gamma.
  \label{eq:privateboundlower}
\end{align}

Let $T_a(s)$ be the number of times arm a is played in the first $s$ time steps.
Let's  $\cbound_{t,n} \defn \sqrt{(2\ln t)/n}$ denote the original UCB  confidence index.

By following similar steps as in the demonstration of UCB in \cite{finitetimemab}, we have
\begin{align}
 T_a(s) &= 1 + \sum_{t = K+1}^{s} \{ a_t  = a \}
 \notag \\
 &\leq \ell +  \sum_{t=1}^{\infty} \sum_{n = 1}^{t-1} \sum_{n_a = \ell}^{t-1}    \notag \\
 &\quad \left \{ \pMean^*_n + \cbound_{t,n} + \dpbound_{n} \leq  \pMean_{a,n_a} + c_{t,n_a} + h_{n_a} \right \}
  \label{eq:num_arms}
\end{align}

In equation \ref{eq:num_arms},  $\pMean^*_n$ is the mean returned by the private mechanism for the best arm when it has been played $n$ times.
Now we can observe that $\pMean^*_n + c_{t,n} + h_{n} \leq  \pMean_{a,n_a} + c_{t,n_a} + h_{n_a}$  implies that at least one of the following must hold
\begin{align}
  \pMean^*_n &\leq {\Mean}^* - c_{t,n} - h_{n}
  \label{eq:cond1}
  \\
  \pMean_{a,n_a} &\geq {\Mean}_a +  c_{t,n_a} + h_{n_a}
  \label{eq:cond2}
  \\
  {\Mean}^{*} &< {\Mean}_a + 2c_{t,n} + 2h_{n}
  \label{eq:cond3}
\end{align}

We can bound the probability of events (\ref{eq:cond1}) using equation  (\ref{eq:privateboundlower}), the union bound and the Chernoff-Hoeffding bound.
\begin{align}
  \Pr(\ref{eq:cond1})  &= \Pr(\pMean^*_n \leq \Mean^* - c_{t,n} - h_n) \notag \\
  &=\Pr(\pMean^*_n \leq \eMean^*_n - h_{n} \lor 
  \eMean^*_n \leq \Mean^* - c_{t,n})
 \notag \\
  &\leq
  \Pr(\pMean^*_n \leq \eMean^*_n - h_{n})
  + \Pr(\eMean^*_n \leq \Mean^* - c_{t,n})
  \notag \\
  &\leq
  \gamma + \exp(-4 \log t)  = \gamma + t^{-4}.
\end{align}
Similarly, to prove a bound on the probability \eqref{eq:cond2} we use \eqref{eq:privateboundupper} , the union bound and the Chernoff-Hoeffding bound.
\begin{align}
  \Pr(\ref{eq:cond2})  &= \Pr(\pMean_{a,n_a} \geq \Mean_a + c_{t,n_a} + h_{n_a}) \notag \\
  &= \Pr(\pMean_{a,n_a} \geq \eMean_{a,n_a} + h_{n_a}
  \lor \eMean_{a,n_a} \geq \Mean_a + c_{t,n_a})
  \notag \\
  &\leq
  \Pr(\pMean_{a,n_a} \geq \eMean_{a,n_a} + h_{n_a})
  + \Pr(\eMean_{a,n_a} \geq \Mean_a + c_{t,n_a})
  \notag \\
  &\leq \gamma + \exp(-4 \log t)  = \gamma + t^{-4} .
\end{align}

Let's choose $\gamma = t^{-4}$; this leads respectively to 

\begin{align}
\Pr(\ref{eq:cond1}) 
&\leq 2t^{-4} \\
\Pr(\ref{eq:cond2}) 
&\leq 2t^{-4}
\end{align}

Now consider the last condition \eqref{eq:cond3}. For this,  we want to find the minimum number $n$ for which event (\ref{eq:cond3}) is always false. Event (\ref{eq:cond3}) is false, implies that $\Delta_a > 2c_{t,n} + 2h_{n}$ where $\Delta_a = \mu_*-\mu_a$.
We observe that for $\Delta_a > 2c_{t,n} + 2h_{n}$ to hold, it is enough that the following two conditions hold
for any  $\lambda_0$ such that $0 < \lambda_0 < 1$.
\begin{align}
  \lambda_0 \Delta_a &> 2c_{t,n}
  \label{eq:cond3ucb}
  \\
  (1-\lambda_0) \Delta_a &> 2h_{n}
  \label{eq:cond3private}
\end{align}

From inequality (\ref{eq:cond3ucb}), we have 
\begin{align}
n \geq  \frac{8}{\lambda_0^2\Delta_a^2} \log(t).
\label{eq:firstsolution}
\end{align} Equation (\ref{eq:cond3private}) leads to  
\[
n \geq B \cdot \log(n) + B
\]

with
\begin{align*}
B &= \frac{\sqrt{8}}{\epsilon(1-\lambda_0)\Delta_a}     \cdot \ln{\frac{4}{\gamma}} \\
  &= \frac{\sqrt{8}}{\epsilon(1-\lambda_0)\Delta_a}  \cdot \ln(4t^{4}) 
\end{align*}

We can rewrite the inequality in a more familiar form:
\[
 e^{-(-\frac{1}{B}) n} \geq e \cdot n
\]

This is a standard transcendental algebraic inequality whose solution is given by the Lambert $W$ function.
So, 
\begin{align*}
n \geq -B \cdot W \left(\frac{-1}{e \cdot B}, -1 \right)   
\end{align*}

where $W(x, k)$ is the Lambert function of $x$ on branch $k$.
Note here that the branch is -1 and because $\frac{-1}{e} < \frac{-1}{e \cdot B} < 0 \quad \forall t > 1$, we are always guaranteed to find a real number.

By using the approximation of the Lambert function provided in section $3.1$ of \cite{Barry200095}, we can conclude that

\begin{align}
n &\geq -B \cdot W \left(\frac{-1}{e \cdot B}, -1 \right) \notag \\
&\approx -B \left(\ln(\frac{1}{e \cdot B}) - \frac{2}{0.3361} \right) \notag \\
&\geq B\left(\ln{B} +7 \right) \notag \\
&=  \frac{\sqrt{8} \cdot \ln(4t^{4})}{\epsilon(1-\lambda_0)\Delta_a}     \left(\ln\left(\frac{\sqrt{8} \cdot \ln(4t^{4})}{\epsilon(1-\lambda_0)\Delta_a}    \right) +7 \right) 
\label{eq:secondsolution}
\end{align}

Combining inequalities (\ref{eq:firstsolution}) and (\ref{eq:secondsolution}) yields, 
\begin{align*}
	n &\geq \max\left[B  \left(\ln\left(B    \right) +7 \right) , \frac{8}{\lambda_0^2\Delta_a^2} \log(t)\right]
\end{align*}
 
In summary,

\begin{align}
  T_a(s) &\leq \ceil*{\max\left[B  \left(\ln{B}  +7 \right), \frac{8}{\lambda_0^2\Delta_a^2} \log(t)\right] } + \notag \\ 
  &\sum_{t=1}^{\infty} \sum_{n = 1}^{t-1} \sum_{n_a = \ell}^{t-1} 4t^{-4}
  \notag \\
  &\leq 1+ \max\left(B  \left(\ln B  +7\right), \frac{8}{\lambda_0^2\Delta_a^2} \log{t}\right) + \sum_{t=1}^{\infty} 4t^{-2}\notag \\
  &\leq 1+ \max\left(B  \left(\ln B  +7\right), \frac{8}{\lambda_0^2\Delta_a^2} \log{t}\right) + \frac{2\pi^2}{3}
  \label{eq:final_bound}
\end{align}
which concludes the proof.

\subsection{Proof for Theorem \ref{theo:bound2}}

The proof of this theorem is similar to the one for theorem \ref{theo:regret_interval} with $\dpbound_n = \frac{1}{\epsilon} \cdot   \sqrt{8}(\log{n}) \cdot \ln{\frac{4}{\gamma}} \cdot \frac{1}{n} +  \frac{1}{\epsilon} \cdot   \sqrt{8} \cdot \ln{\frac{4}{\gamma}} \cdot \frac{1}{n}$. We make the same choice of $\lambda$ and $\gamma$. However, the minimum number $n$ compatible with these choices lead a transcendental equations in the same form as the one in the proof of theorem \ref{theo:bound}. 

\section{Proofs for \textit{UCB-Interval} Algorithm}
\begin{fact}{Differential privacy of the Laplace mechanism~(See Theorem 4 in \cite{dpbook})}
  For any real function $g$ of the data, a mechanism adding Laplace noise with scale parameter $\beta$ is $\Delta g / \beta $-differentially private, where $\Delta g$ is the $L1$ sensitivity of $g$.
  \label{fact:dp-laplace}
\end{fact}

\subsection{Proof of Lemma \ref{lemma:interval_dp_step}}

\begin{proof}
Indeed, for each arm, we are adding a Laplace noise of mean 0 and scale $\nat^{v/2-1}$ where $\nat$ is the number of times this arm has been played. As the sensitivity of the mean is $\frac{1}{\nat}$, we use the differential privacy of the Laplace Mechanism (Fact~\ref{fact:dp-laplace}) to conclude the proof.
\end{proof}

\subsection{Proof of Theorem   \ref{theo:interval_simple_privacy}  }

Similarly to the proof of Theorem \ref{theo:privacy_algo1},
the overall privacy of Algorithm \ref{alt2} will be the same as the overall privacy of the Mechanism computing the mean of the rewards received from a single arm. 

A new Laplace Mechanism is used to compute the mean of each arm. However a new mean is only released $t/f$ times (every $f$ time steps) after $t$ times steps where $f$ is the interval used. 

\begin{comment}
First we can see Algorithm \ref{alt2} as releasing a randomized
function of the means computed for each arm. According to the
standard post processing theorem for differential private algorithm
\cite{dpbook}, the output of Algorithm~\ref{alt2} at any step will
have the same privacy as the composition of the computed means.

Now, let's observe that when a single reward is changed in the
sequence of the rewards, only one arm is affected. So the privacy of
Algorithm \ref{alt2} will be the composition of the released means
by the affected arm. And if we consider lemma
\ref{lemma:interval_dp_step}, we can see that privacy will be worse
(higher) when the affected arm is played more often.

However, the algorithm only computes the private mean at certain
intervals. Between two successive updates, actions are taken using
only the previously computed values for the mean and the number of
times the action was taken. So, we will incur no privacy loss
between two updates. In other word, we are 0-differential private
between two updates. In conclusion, Algorithm \ref{alt2} will only
release $T/f$ where $f$ is the interval used to release new updated
means.
\end{comment}

  According to the $k$-fold adaptive composition theorem (III-3 in \cite{dworkrv10}), Algorithm \ref{alt2} will be 
$(\epsilon', \delta')$ differential private for any $\delta' \in (0, 1]$ with
\[\epsilon' \leq \min \left\{C,D \right\}
\]

\[C = \sum_{n \in \mathcal{I}_{\frac{1}{\epsilon}}^{n_a}} \frac{1}{\sqrt{n^{v}}}\]

\[D = \sum_{n \in \mathcal{I}_{\frac{1}{\epsilon}}^{n_a}} \frac{1}{\sqrt{n^{v}}} (e^{\frac{1}{\sqrt{n^{v}}}}-1) + \sqrt{\sum_{n \in \mathcal{I}_{\frac{1}{\epsilon}}^{n_a}} \frac{2}{n^{v}} \ln(1/\delta')} \]

where $\mathcal{I}_{\frac{1}{\epsilon}}^{n_a} = \{ {\frac{1}{\epsilon}, \frac{2}{\epsilon}, \frac{3}{\epsilon}, \cdots n_a} \}$

We have

\begin{align}
D & = \left( \sum_{n \in \mathcal{I}_{\frac{1}{\epsilon}}^{n_a}} \frac{1}{\sqrt{n^{v}}} (e^{\frac{1}{\sqrt{n^{v}}}}-1) + \sqrt{\sum_{n \in \mathcal{I}_{\frac{1}{\epsilon}}^{n_a}} \frac{2}{n^{v}} \ln(1/\delta')} \right)  \\
&\leq \left( \sum_{n \in \mathcal{I}_{\frac{1}{\epsilon}}^{t}} \frac{1}{\sqrt{n^{v}}} (e^{\frac{1}{\sqrt{n^{v}}}}-1) + \sqrt{\sum_{n \in \mathcal{I}_{\frac{1}{\epsilon}}^{t}} \frac{2}{n^{v}} \ln(1/\delta')} \right)\label{line:adapt0}
\\
&\leq \left(\epsilon \sum_{n=1}^{t} \frac{1}{\sqrt{n^{v}}} (e^{\frac{1}{\sqrt{n^{v}}}}-1) + \sqrt{\epsilon\sum_{n=1}^{t} \frac{2}{n^{v}} \ln(1/\delta')} \right)\label{line:adapt1}
\end{align}

\begin{align}
	C &= \sum_{n \in \mathcal{I}_{\frac{1}{\epsilon}}^{n_a}} \frac{1}{\sqrt{n^{v}}}\\
	&\leq \sum_{n \in \mathcal{I}_{\frac{1}{\epsilon}}^{t}} \frac{1}{\sqrt{n^{v}}} \label{line:simple0}\\
	&\leq \epsilon \sum_{n =1}^{t} \frac{1}{\sqrt{n^{v}}}\label{line:simple1}
\end{align}

which concludes the proof.

\begin{lemma}[Values of the interval of release of the means]
	The interval $f$ by which we should update the mean follows a series such that $f_n = W_{n+1} - W_n$ with:
	\begin{align*}
			\left\{ \begin{array}{ll}
				W_0 = 0 \\
				%W_{n+1} = x  \, \forall t \leq T \; \text{ or } \; W_{n+1} = y \, \forall t \leq T\\
				W_{n+1} \text{ is either } x \text{ or } y \text{ with }\\
				x = \displaystyle \inf_{x' \in \mathbb{N}} \left\{x' \geq W_n + 1 : \sum_{i=W_{n} + 1}^{x'} \frac{1}{\sqrt{i^v}} \geq \frac{1}{\epsilon \sqrt{{x'}^{v}}} \right\}\\				
				y = 	\displaystyle \inf_{y' \in \mathbb{N}} \left\{y' \geq W_n + 1 : \sum_{i=W_{n} + 1}^{y'} \frac{1}{{i^v}} \geq \frac{1}{\epsilon {y'}^{v}} \right\}		
				%y =  \displaystyle \inf_{y \in \mathbb{N} } \left\{y > W_n + 2: \sum_{i=W_{n} + 1}^{y-1} \frac{e^{\frac{1}{\sqrt{i^v}} }-1}{\sqrt{i^v}} \right. \\
				%\displaystyle \left. \qquad \qquad \qquad \qquad+ \sqrt{ \sum_{i=W_{n} + 1}^{y-1} \frac{2\ln(\frac{1}{\delta'})}{i^v}} \geq \frac{1}{\epsilon y^{v/2}} \right\}
			\end{array} \right.
	\end{align*}
	
	Furthermore $f_n$ is always such that $f_n \leq \ceil*{\frac{1}{\epsilon}}$
	
	Note that only one of the expression for $W_{n+1}$ should be used up to the horizon T.
	
	\label{interval-frequency}
\end{lemma}

\begin{proof}
	From the proof of theorem \ref{theo:interval_simple_privacy}, we can easily see that we can pass from lines \eqref{line:adapt0} to \eqref{line:adapt1} if the conditions on $y$ in Lemma \ref{interval-frequency} is verified.
	Similarly, we can pass from lines \eqref{line:simple0} to \eqref{line:simple1} if the conditions on $y$ in Lemma \ref{interval-frequency} is verified. In both cases, $\frac{1}{\epsilon}$ is an upper bound on the number $x - W_n$ and $y - W_n$ as the original interval used in both lines \eqref{line:adapt0} and \eqref{line:simple0} is $\frac{1}{\epsilon}$.
\end{proof}

\subsection{Proof of Corollary \ref{cor:alt2:dp}}
	
From theorem \ref{theo:interval_simple_privacy}, we know that Algorithm \ref{alt2} will be 
$(\epsilon', \delta')$ differential private for any $\delta' \in (0, 1]$, $ 1 < v \leq 1.5$ with
\[\epsilon' \leq \min \left\{C,D \right\}
\]

\[ C = \sum_{n=1}^{t} \frac{\epsilon}{\sqrt{n^{v}}}\]
\[ D = \epsilon \sum_{n=1}^{t} \frac{e^{\frac{1}{\sqrt{n^{v}}}}-1}{\sqrt{t^{v}}}  + \sqrt{\epsilon\sum_{n=1}^{t} \frac{2\ln\frac{1}{\delta'}}{n^{v}} }\]

It is easy to get an upper bound for $C$ using the integral test inequality giving:

$C \leq  \epsilon \frac{t^{1-v/2} -v/2}{1-v/2}$

We will now simplify $D$ using a standard approximation for exponential function.

\begin{align}
D & =  \epsilon \sum_{n=1}^{t} \frac{e^{\frac{1}{\sqrt{n^{v}}}}-1}{\sqrt{t^{v}}}  + \sqrt{\epsilon\sum_{n=1}^{t} \frac{2\ln\frac{1}{\delta'}}{n^{v}} }\\
&\leq \left(\epsilon \sum_{t=1}^{\infty} \frac{1}{\sqrt{t^{v}}} (e^{\frac{1}{\sqrt{t^{v}}}}-1) + \sqrt{\epsilon \sum_{t=1}^{\infty} \frac{2}{t^{v}} \ln(1/\delta')} \right)\\
&\leq  \left(\epsilon \sum_{t=1}^{\infty} 2 (\frac{1}{\sqrt{t^{v}}})^2 + \sqrt{\epsilon\sum_{t=1}^{\infty} \frac{2}{t^{v}} \ln(1/\delta')} \right)\\
&= \left( 2 \epsilon \zeta(v) + \sqrt{2 \epsilon \zeta(v) \ln(1/\delta')} \right)
\end{align}

where $ \zeta$ is the Riemann zeta function.

\subsection{Proof of Corollary \ref{alt2_cor}}

The proof is immediate from corollary \ref{cor:alt2:dp}. It is obtained by inverting the term using the Riemann zeta function in corollary  \ref{cor:alt2:dp}.

\subsection{Proof of Theorem \ref{theo:regret_interval}}

\begin{comment}
\begin{fact}[Corollary 2.9 in~\cite{chan2010private}]
If $\lambda_i \sim \Laplace(b_i)$ and $Y = \sum_i \lambda_i$ then
\begin{equation}
\Pr(|Y| \geq \|b\|_2 \ln \frac{1}{\delta}) \leq \delta.
\label{eq:laplace}
\end{equation}
\end{fact}
\end{comment}

\begin{fact}[Fact 3.7 in~\cite{dpbook}]
	If $Y \sim \Laplace(b)$ then
	\begin{equation}
	\Pr(|Y| \geq b \ln \frac{1}{\gamma}) = \gamma.
	\end{equation}
	\label{fact:single-laplace}
\end{fact}

  First, let's ignore the effect of computing the means per interval
  and assume that we compute it at each time step after playing each
  arm $f_0$ times where $f_0 \leq
  \frac{1}{\epsilon}$ is the first value of the series $f$ defined in lemma \eqref{interval-frequency}.  As much of the proof is quite similar to that of 
  Theorem~\ref{theo:bound}, we shall omit some steps.

Similarly to the proof of Theorem \ref{theo:bound}, we will take a bad arm when one of the following 3 events happens:
\begin{align}
  \pMean^*_n &\leq {\Mean}^* - c_{t,n}
  \label{eq:cond1_alt2}
  \\
  \pMean_{a,n_a} &\geq {\Mean}_a +  c_{t,n_a}
  \label{eq:cond2_alt2}
  \\
  {\Mean}^{*} &< {\Mean}_a + 2c_{t,n}
  \label{eq:cond3_alt2}
\end{align}

Since we are adding Laplace noise, we can use Fact~\ref{fact:single-laplace} to show that:
\begin{align}
  \Pr\{\abs{Y-X} \geq \log(\frac{1}{\gamma}) n^{v/2-1}\} \leq \gamma
  \label{eq:privatebound_alt2}
\end{align}
Let $h_n = \log(\frac{1}{\gamma}) n^{v/2-1}$. Now we can bound the probability of events (\ref{eq:cond2_alt2}) using equation  (\ref{eq:privatebound_alt2}), the union bound and the Chernoff-Hoeffding bound. In order to be more precise, we use  $\pMean_{a, n_a}$ for $\hat{x}_a$ and $\eMean$ to denote the empirical mean $\bar{x}_a$ at time $t$, while $c_{t,n_a} = \sqrt{2 \log(t) / n_a}$ as before.
\begin{align}
  \Pr(\ref{eq:cond2_alt2}) &= \Pr(\pMean_{a,n_a} \geq \Mean_a + c_{t,n_a}) \notag \\
  &=
  \Pr(\pMean_{a,n_a} \geq \eMean_{a,n_a} + h_{n_a}
  \lor \eMean_{a,n_a} \geq \Mean_a + c_{t,n_a} - h_{n_a})
   \notag \\
  &\leq
  \Pr(\pMean_{a,n_a} \geq \eMean_{a,n_a} + h_{n_a})
  + \Pr(\eMean_{a,n_a} \geq \Mean_a + c_{t,n_a} - h_{n_a})
  \notag \\
  &\leq \gamma +\Pr(\eMean_{a,n_a} \geq \Mean_a + c_{t,n_a} - h_{n_a}) .
  \\
  &\leq \gamma + \exp(-2n_a (c_{t,n_a} - h_{n_a})^2)
  \\
  & \leq \gamma + t^{-3.5}
  \\
  &\leq 2t^{-3.5}
\end{align}
In the above, we choose $h_{n_a} \leq \lambda_4 c_{t,n_a} $
with $\lambda_4 = 1 - \frac{\sqrt{3.5}}{2}$ and
$\gamma = t^{-3.5}$

Similarly, we prove a bound on the probability \eqref{eq:cond1_alt2}:
\begin{align}
  \Pr(\ref{eq:cond1_alt2}) &= \Pr(\pMean^*_n \leq \Mean^* - c_{t,n}) \notag \\
  &
  =
  \Pr(\pMean^*_n \leq \eMean^*_n - h_{n_a} \lor 
  \eMean^*_n \leq \Mean^* - c_{t,n} + h_{n_a})
  \notag \\
  &\leq
  \Pr(\pMean^*_n \leq \eMean^*_n - h_{n_a})
  + \Pr(\eMean^*_n \leq \Mean^* - c_{t,n} + h_{n_a} )
 \notag  \\
  &\leq
  \gamma + t^{-3.5}  = 2t^{-3.5}.
\end{align}

Now, Let's compute the minimum value for $n_a$ which is consistent with our choice of $\lambda_4$ and $\gamma$. It is easy to show that  $n_a \geq (\frac{2 - \sqrt{3.5}}{3.5\sqrt{2\log t}})^{\frac{2}{v-1}}$. And this number converges to 0 as t increases.

The final event \ref{eq:cond3_alt2} is exactly the standard UCB event and \ref{eq:cond3_alt2} will be false for all
$n_a \geq \frac{8}{\Delta_a^2} \log t$.

Now, we can easily see that the effect of the interval in the algorithm will not change this number. Indeed, because we are only updating the mean each $f$ steps, the number $n_a$ should be divided by $f$. Now, after updating the mean, we play this arm for a number of steps lower or equal to $f$. So, we should multiply $n_a$ by $f$ which cancel out the effect of the division.

In summary, (and similarly to Theorem \ref{theo:bound})
\begin{align}
  T_a(s) &\leq f_0 + \ceil*{ \max\left((\frac{2 - \sqrt{3.5}}{3.5\sqrt{2\log t}})^{\frac{2}{v-1}}, \frac{8}{\Delta_a^2} \log t\right) } + \notag \\
  & \qquad \sum_{t =1}^{\infty} \sum_{n = 1}^{t-1} \sum_{n_a =\ell}^{t-1} 4t^{-3.5}
  \notag 
\\
   &\leq f_0+ \ceil*{ \frac{8}{\Delta_a^2} \log t} +   \sum_{t=1}^{\infty}  4t^{-1.5}
 \notag \\ 
  &\leq f_0 + 1 + \frac{8}{\Delta_a^2} \log t + 4 \zeta(1.5)
  \notag \\ 
  &\leq f_0 + 1 +\frac{8}{\Delta_a^2} \log t + 4 \zeta(1.5)
  \label{eq:final_bound_alt2}
\end{align}

\begin{lemma}[Improved~\cite{chan2010private} bound]
  For any $\gamma \leq n^{-b}$, where $b > 0$, Chan's hybrid mechanism is $\epsilon$-differential private and has an error bounded with probability at least $1 - \gamma$ by $\frac{\sqrt{8}}{\epsilon} \log(\frac{4}{\gamma}) \log n + \frac{\sqrt{8}}{\epsilon} \log(\frac{4}{\gamma})$.
\label{lem:chan-improved}
\end{lemma}
\begin{proof}
Our improvement is based on two observations. Firstly, if $\gamma \leq n^{-b}$ then we can use the alternative concentration bound for the sum of the Laplace distribution. Secondly, we note that given all the previous released sums, the current sum is only drawn from one of the two mechanisms. Thus there is no need to apply a composition theorem. 
%It is typical to obtain a high probability bound by setting $\delta = \frac{1}{a} \log(n^{-b})$, $a>= 1$, $b \geq 1$. This is the choice we made in our paper. In this case, it is possible to get an overall tighter bound than the more loose bound from 
%So, the Binary Mechanism error bound is:
%$\frac{\sqrt{8}}{\epsilon} \log(\frac{2}{\delta}) \log n$ and the logarithm mechanism error bound is $\frac{\sqrt{8}}{\epsilon} \log(\frac{2}{\delta})$.  
%In \cite{chan2010private}, takes the privacy of the logarithm and binary mechanism to be $\frac{\epsilon}{2}$ so that the overall privacy of the hybrid mechanism is $\epsilon$. However, this is not needed and we can take the privacy of both mechanism as $\epsilon$. Indeed, when $n$ is a power of 2 we only release the logarithm mechanism sum enjoying a privacy of $\epsilon$. When  $n$ is not a power of 2, we use the previously released logarithm mechanism sum (giving no privacy loss or 0-privacy) to which we add the binary mechanism sum (giving a privacy of $\epsilon$).
%In summary the overall error of the hybrid mechanism is 
%$\frac{\sqrt{8}}{\epsilon} \log(\frac{4}{\delta}) \log n + \frac{\sqrt{8}}{\epsilon} \log(\frac{4}{\delta})$ with probability at least $1-\delta$. This gives an $\epsilon$-differential privacy.
\end{proof}

\end{document}